\documentclass{eptcs}

\usepackage{pdfsync}
\usepackage{amsmath,amssymb,latexsym,amsthm}
\usepackage[utf8]{inputenc}
\usepackage{textcomp}
\usepackage{enumerate}
\usepackage{url}
\usepackage{wrapfig}

\newcommand{\titlerunning}{Non-simplifying Graph Rewriting Termination}
\newcommand{\authorrunning}{G. Bonfante and B. Guillaume}
\title {Non-simplifying Graph Rewriting Termination}
\author {
Guillaume Bonfante
\institute{LORIA\\Université de Lorraine}
\and Bruno Guillaume
\institute{LORIA\\Inria Nancy Grand-Est}
}

\date{January 2012}

\usepackage{amsmath}
\usepackage{bbm}
\usepackage{yfonts}
\usepackage{mathtools}

\usepackage[pdftex]{graphicx}

\usepackage{tikz}
\usetikzlibrary{calc,arrows}
\tikzstyle{etat}=[circle, minimum size=1.2cm, draw=black]

\newcommand{\shiftf}{{\Phi}}

\newcommand{\txt}[1]{{\it "#1"}}

\newcommand{\arete}[3]{{#1 \stackrel{#2}{\longrightarrow} #3}}

\renewcommand{\rule}[2]{{\langle #1, #2\,\rangle}}

\newcommand{\dhe}{{{h}}}

\newcommand{\Z}{{\mathbb{Z}}}
\newcommand{\N}{{\mathbb{N}}}
\newcommand{\Id}{{\mathbbm{1}}}

\renewcommand{\bigg}{\textswab{G}}
\newcommand{\prog}{{\cal G}}

\newcommand{\size}[1]{{|#1|}}
\newcommand{\sizevec}[1]{{|\vec{#1}\,|}}

\newcommand{\elabs}{{\Sigma_E}}

\newcommand{\nlabs}{{\Sigma_N}}

\newcommand{\noin}{{\bar{\cal{I}}}}
\newcommand{\noout}{{\bar{\cal{O}}}}
\newcommand{\noedges}{{\bar{\cal{E}}}}

\newif\ifname
\nametrue

\newcommand{\nodes}{{\cal{N}}}
\newcommand{\edges}{{\cal{E}}}

\newcommand{\pattim}{{{\cal P}}}
\newcommand{\crown}{{{\cal K}}}
\newcommand{\context}{{{\cal C}}}
\newcommand{\contextcomp}{{{\cal Q}}}
\newcommand{\glued}{{{\cal H}}}

\newcommand{\ie}{{\it i.e.~}}

\newtheorem{lemma}{Lemma}[section]
\newtheorem{definition}{Definition}[section]
\newtheorem{example}{Example}[section]

\newtheorem{theorem}{Theorem}[section]

\newcommand{\ren}[2]{{\tt label} ({#1},{#2})}
\newcommand{\add}[3]{{\tt add\_edge} ({#1},{#2},{#3})}
\newcommand{\del}[3]{{\tt del\_edge} ({#1},{#2},{#3})}
\newcommand{\shift}[2]{{\tt shift} ({#1},{#2)}}
\newcommand{\delnode}[1]{{\tt del\_node} ({#1})}
\newcommand{\xren}{{\tt label}}
\newcommand{\xadd}{{\tt add\_edge}}
\newcommand{\xdel}{{\tt del\_edge}}
\newcommand{\xshift}{{\tt shift}}
\newcommand{\xdelnode}{{\tt del\_node}}

\newcommand\restr[2]{\ensuremath{\left.#1\right|_{#2}}}

\begin{document}

\label{firstpage}
\maketitle

\begin{abstract}
So far, a very large amount of work in Natural Language Processing (NLP) rely on trees as the core mathematical structure to represent linguistic informations (e.g. in Chomsky's work).
However, some linguistic phenomena do not cope properly with trees.
In a former paper, we showed the benefit of encoding linguistic structures by graphs and of using graph rewriting rules to compute on those structures.
Justified by some linguistic considerations, graph rewriting is characterized by two features: first, there is no node creation along computations and second, there are non-local edge modifications.
Under these hypotheses, we show that uniform termination is undecidable and that non-uniform termination is decidable.
We describe two termination techniques based on weights and we give complexity bound on the derivation length for these rewriting systems.
\end{abstract}



\section{Introduction}

Linguists introduce different levels to describe a natural language sentence.
Starting from a sentence given as a sequence of sounds or as a sequence of words; among the linguistic levels, two are deeply considered in literature:
the syntactic level (a grammatical analysis of the sentence) and the semantic level (a representation of the meaning of the sentence).
These two representations involve mathematical structures such as logical formulae, $\lambda$-terms, trees and graphs.

\begin{wrapfigure}{r}{68mm}
\includegraphics[scale=.35]{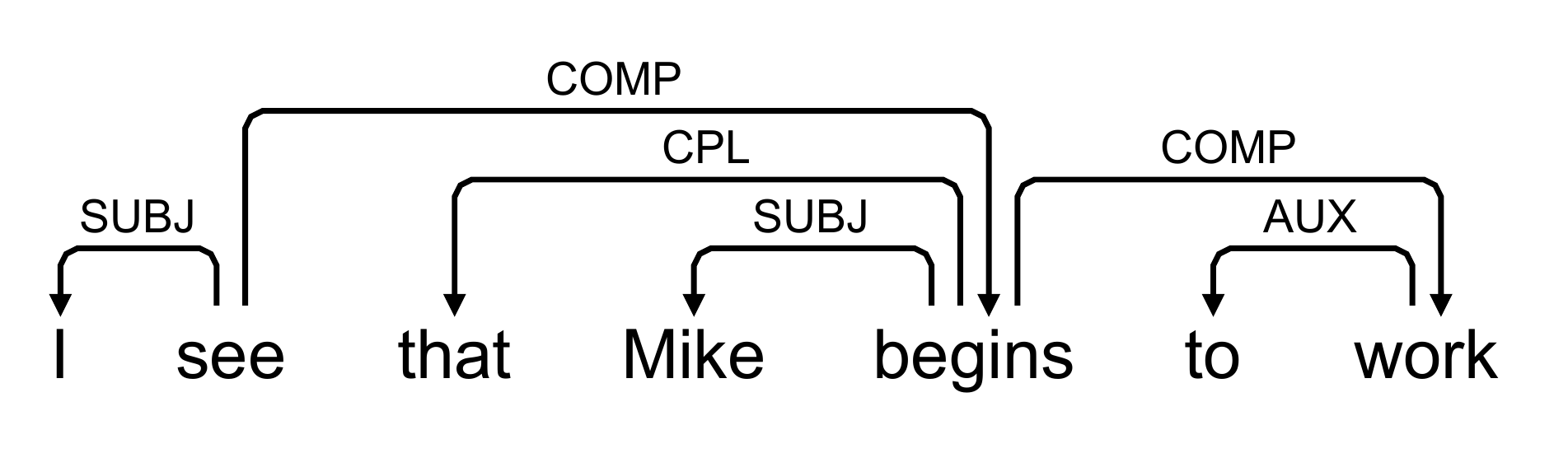}
\end{wrapfigure}
One of the usual ways to describe syntax is to use the notion of dependency~\cite{elements}.
A dependency structure is an ordered sequence of words, together with some relations between these words.
For instance, the sentence \txt{I see that Mike begins to work} can be represented by the structure on the right.

There is a large debate in the literature about the mathematical nature of the structures needed for natural language syntax: do we have to consider trees or graphs?
Trees are often considered for their simplicity; however, it is clearly insufficient.
Let us illustrate the limitations of tree-representations with some linguistic examples.
Consider the sentence \txt{Bill expects Mary to come}, the node \txt{Mary} is shared, being the subject of \txt{come} and the object of \txt{expects} (below on the left).
The situation can be even worse: cycles may appear such as in the sentence below where edges in the cycle are drawn with dashed line (below on the right).

\includegraphics[scale=.35]{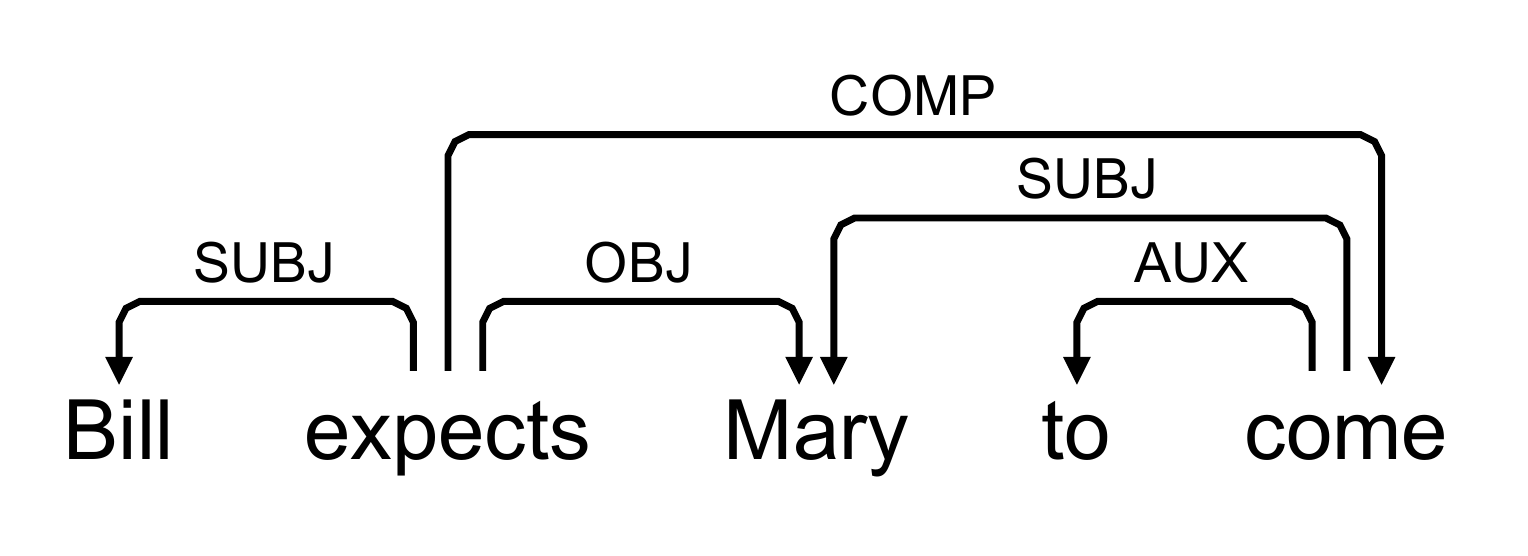} \qquad \includegraphics[scale=.35]{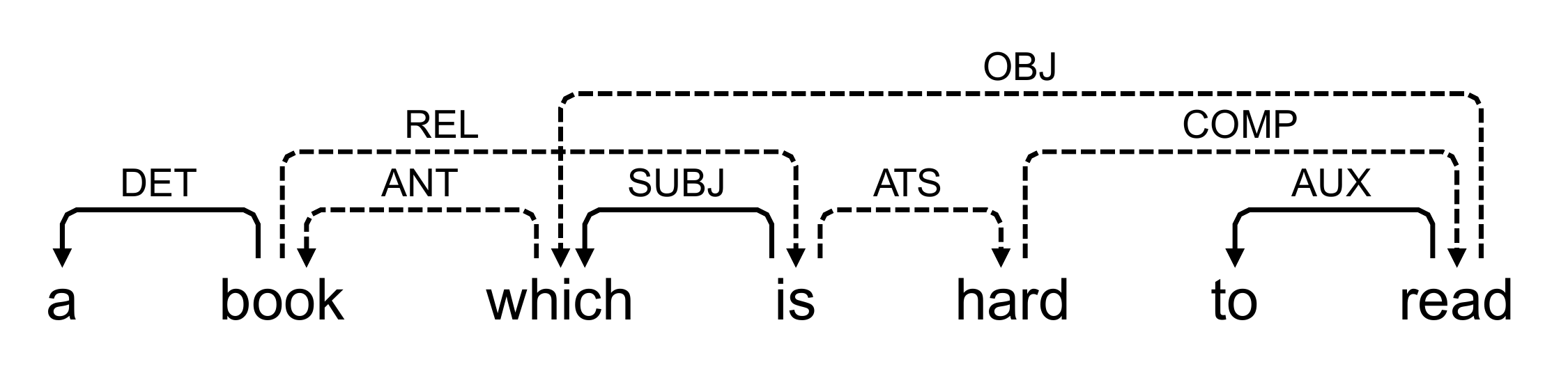}

\newpage
\begin{wrapfigure}{r}{96mm}
\begin{tikzpicture}[->,transform shape, minimum size=10, scale=0.9]]
\tikzstyle{dmrs}=[shape=rectangle, draw=black]
\node[dmrs] (byte) at (5,5) {\sc{byte}$_v$};

\node[dmrs] (the2) at (2,5) {\sc{the}$_q$};
\node[dmrs] (cat) at (3,3) {\sc{cat}$_n$};
\path (the2) edge node[sloped, above] {\tiny RSTR} (cat);

\node[dmrs] (the1) at (2,2) {\sc{the}$_q$};
\node[dmrs] (dog) at (3,0) {\sc{dog}$_n$};
\path (the1) edge node[sloped, above] {\tiny RSTR} (dog);

\node[dmrs] (bark) at (0,2) {\sc{bark}$_v$};
\path (bark) edge node[sloped, above] {\tiny ARG1} (dog);

\node[dmrs] (poss) at (5,2) {\it{poss}$_v$};
\path (poss) edge [bend right=10] node[sloped, above] {\tiny ARG2} (dog);
\node[dmrs] (toy) at (8,0) {\sc{toy}$_n$};
\path (poss) edge node[sloped, above] {\tiny ARG1} (toy);
\path (byte) edge node[sloped, above] {\tiny ARG1} (cat);
\path (byte) edge node[sloped, above] {\tiny ARG2} (toy);
\path (byte) edge [-,dashed] node[sloped, above] {\tiny EQ} (dog);
\path (poss) edge [-,dashed, bend left=10] node[sloped, below] {\tiny EQ} (dog);

\node[dmrs] (def) at (9,2) {\it{def}$_q$};
\path (def) edge node[sloped, above] {\tiny RSTR} (toy);
\end{tikzpicture}
  \label{fig:dmrs}
 \end{wrapfigure}
For the semantic representation of natural language sentences, first order logic formulae are widely used.
To deal with natural language ambiguity, a more compact representation of a set of logic formulae (called underspecified semantic representation) is used.
DMRS~\cite{copestake} is one of these compact representation.
The DMRS structure for the sentence \txt{The Dog whose toy the cat bit barked} is given in the figure on the right.

To describe transformations between syntactic and semantics structures, there are solutions based on many computational models (finite state automata, $\lambda$-calculus).
It is somewhat surprising that Graph Rewrite Systems (GRS) have been hardly considered so far (\cite{Hyvonen84, Bohnet01, Crouch05, Jijkoun07}).
To explain that, GRS implementations are usually considered to be too inefficient to justify their extra-generality.
For instance, pattern matching does not take linear time where this is usually seen as an upper limit for fast treatment.

However, if one drops for a while the issue of efficiency, the use of GRS is promising.
Indeed, linguistic considerations can be most of the time expressed by some relations between a few words.
Thus, they are easily translated into rules.
To illustrate this point, in~\cite{iwcs11, treebanking12}, we proposed a syntax to semantics translator based on GRSs: given the syntax of a sentence, it outputs the different meaning associated to this syntax.

In the two earlier mentioned studies, we tried to delineate what are the key features of graph rewriting in the context of NLP.
Roughly speaking, node creation are strictly restricted, edges may be shifted from one node to another and there is a need for negative patterns.
Based on this analysis, we define here a suitable framework for NLP (see Section~\ref{sec:def}).

Compared to term rewriting, the semantics of graph rewriting is problematic: different choices can be made in the way the context is glued to the rule application~\cite{dpo_spo}.
As far as we see, our notion does not fit properly the DPO approach due to unguarded node deletion nor the SPO approach due to the shift command, as we shall see.
Thus we will provide a complete description of our notion.
We have chosen to present it in an operational way and we leave for future work a categorial semantics.

In our application, we use several hundreds of rules.
To manage such a system, we use a notion of modular graph rewriting system: the full set of rules is divided in smaller subsets (called modules) that are used in turn.

In practice, we need some tools to verify termination and confluence properties of modules.
In Section~\ref{sec:term}, we provide two termination methods based on a weight analysis.
First, there is a direct motivation: in our NLP application, any computations should terminate.
If it is not the case, it means that the rules where not correctly drawn.
Then, termination ensures partly the correctness of the transformation.
There is also an indirect reason to consider termination: one way of establishing confluence is through Newman's Lemma~\cite{N42} which requires termination.

We consider two properties of the above mentioned termination methods.
First, we show that they are decidable, that is the existence of weights can be computed statically from the rules, and thus we have a fully automatic tool to verify termination.
Obviously, it is not complete.
In a second step, we evaluate the strength of the two methods.
To do that, we consider what restrictions they impose on the length of computations.
We get quadratic time for the first method, polynomial time for the second.
This article is an extended abstract of~\cite{MSCS2013}.
\section{Linguistic motivations}
\label{sec:nlp}

Without any linguistic exhaustivity, we highlight in this section some crucial points of the kind of linguistic transformation we are interested in and hence the relative features of rewriting we have to consider.

\paragraph{Node preservation property.}
As linguistic examples above suggest, the goal of linguistic analysis is mainly to describe different kinds of relations between elements that are present in the input structure.
As a consequence, the set of nodes in the output structure is directly predictable from the input and only a very restrictive notion on node creation is needed.
In practice, these node creations can be anticipated in some enriched input structure on which the whole transformation can be described as a non-size increasing process.

\begin{wrapfigure}[4]{r}{65mm}
\vspace{-7mm}
  \includegraphics[scale=.3]{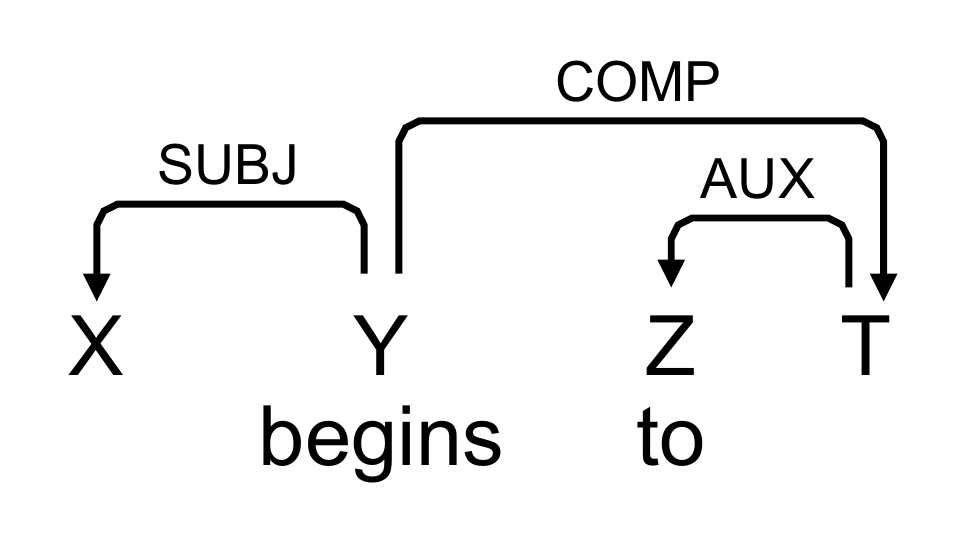}  \quad \includegraphics[scale=.3]{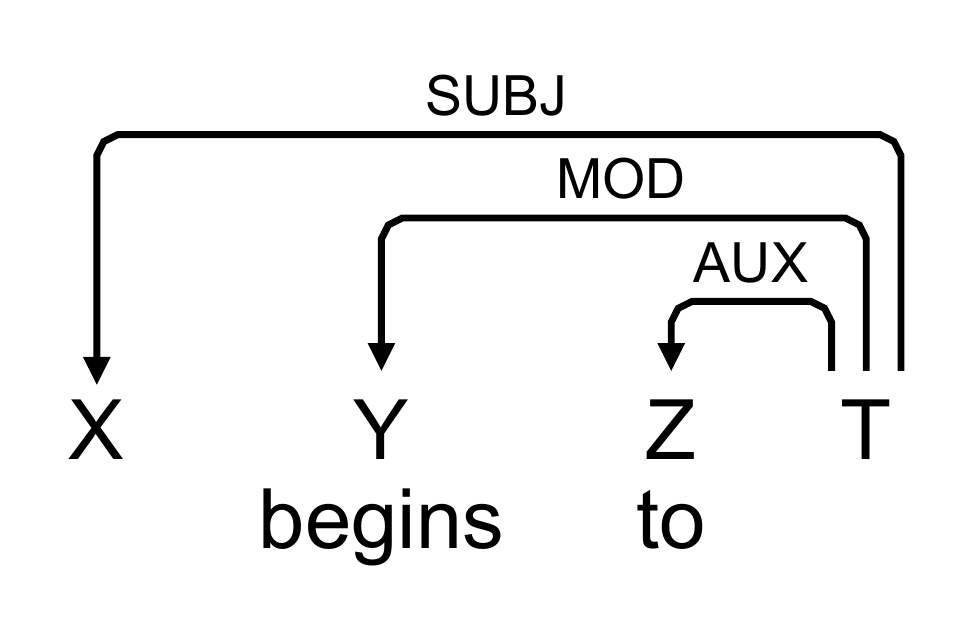}
\end{wrapfigure}
\paragraph{Edge shifting.}
In the first example of the introduction (for the sentence \txt{I see that Mike begins to work}), the verb \txt{begins} is called a raising verb and we know that \txt{Mike} is the \emph{deep subject} of the verb \txt{work}; \txt{begins} being considered as a modifier of the verb.
To recover this deep subject, one may imagine a local transformation of the graph which turns the first graph on the right into the second one.

However, in our example above, a direct application of such a transformation leads to the structure below on the left which is not the right structure.
Indeed, the transformation should shift what the linguists call the head of the phrase \txt{Mike begins to work} from the word \txt{begins} to the word \txt{work} with all relative edges.
In that case, the transformation should produce the structure below on the right:

\begin{center}
  \includegraphics[scale=.35]{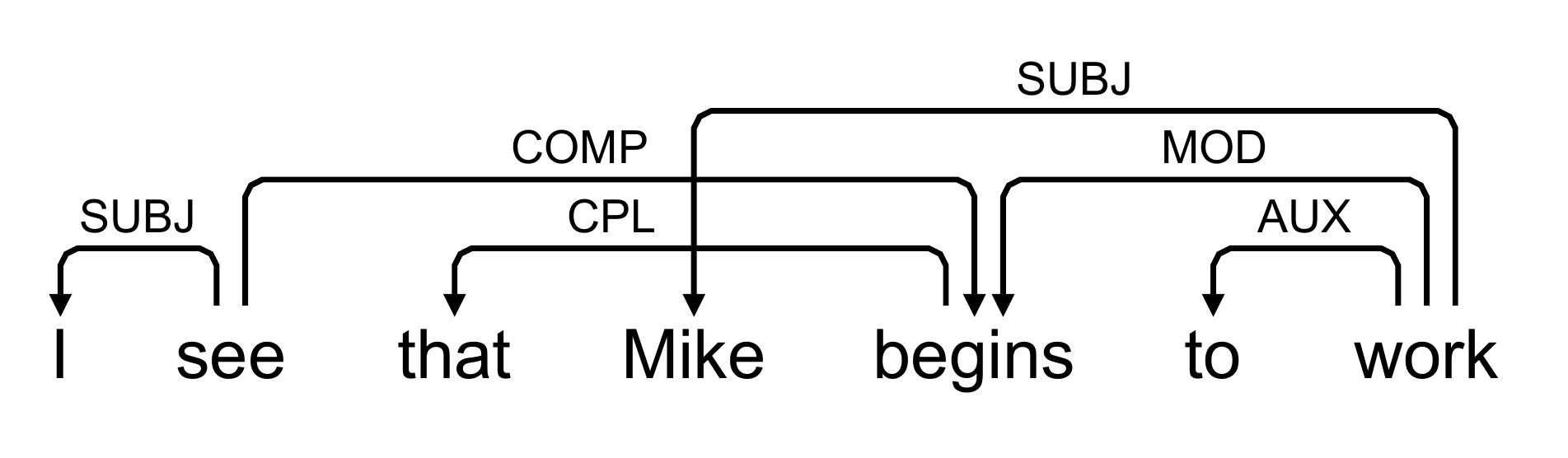} \qquad
  \includegraphics[scale=.35]{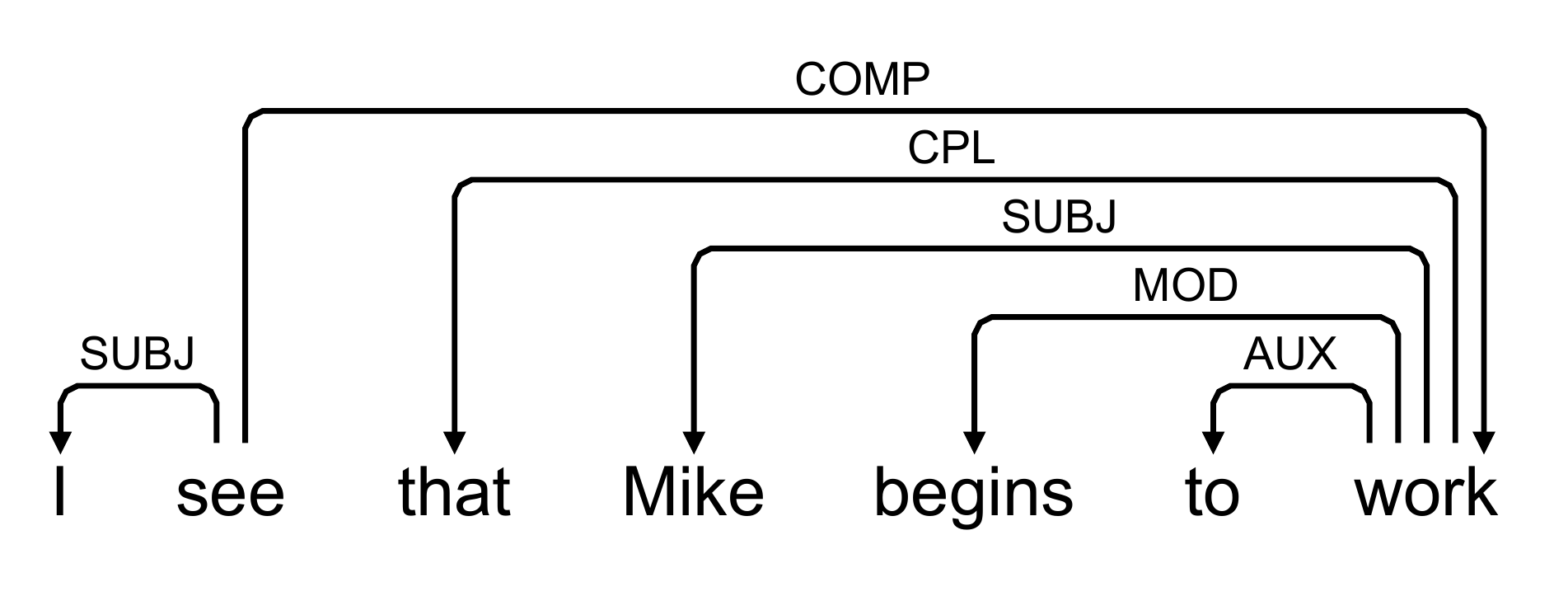}
\end{center}

In a more general setting, our transformations may have to specify the fact that all incident edges of some node $X$ must be transported to some other node $Y$.
We call this operation $\xshift$.

To describe our graph rewriting rules, we introduce a system of commands (like in~\cite{echahed}) which expresses step by step the modifications applied on the input graph.
The transformation described above is performed in our setting as follows:
\medskip
\begin{center}
\begin{tabular}[c]{|c|c|}
  \raisebox{-8mm}{\includegraphics[scale=.35]{pattern_1.pdf}} &
  \begin{tabular}[c]{l}
    \del{Y}{SUBJ}{X};  \del{Y}{COMP}{T}; \\
    \add{T}{SUBJ}{X}; \add{T}{MOD}{Y}; \\
    \shift{Y}{T}
  \end{tabular}
\end{tabular}
\end{center}

\paragraph{Negative conditions.}
In some situation, rules must be aware of the context of the pattern to avoid unwanted ambiguities.
When computing semantics out of syntax, one has to deal with passive sentence; the two sentences below show that the agent is optional.

\begin{center}
  \includegraphics[scale=.35]{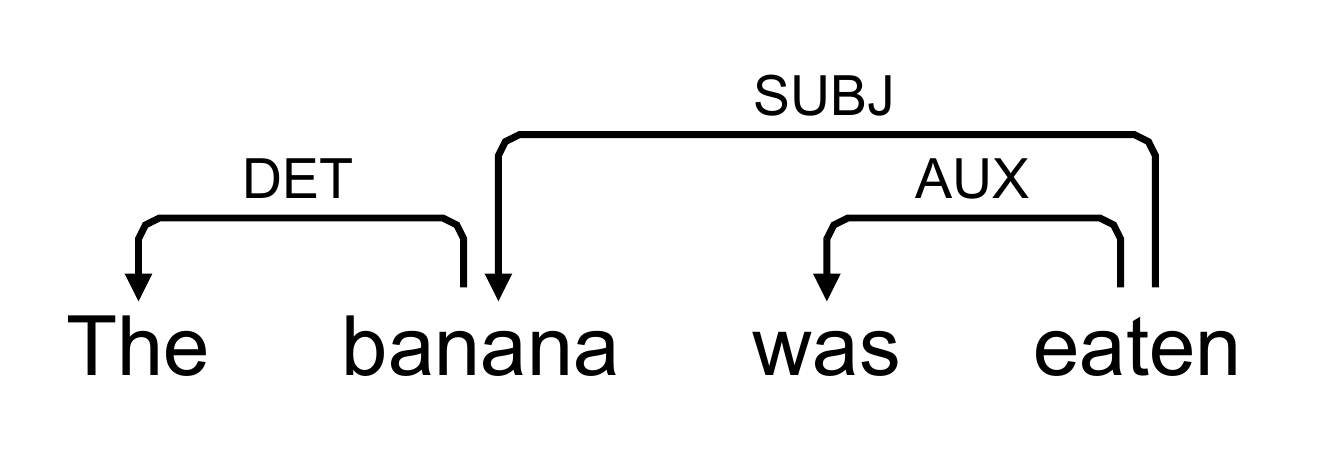} \qquad \includegraphics[scale=.35]{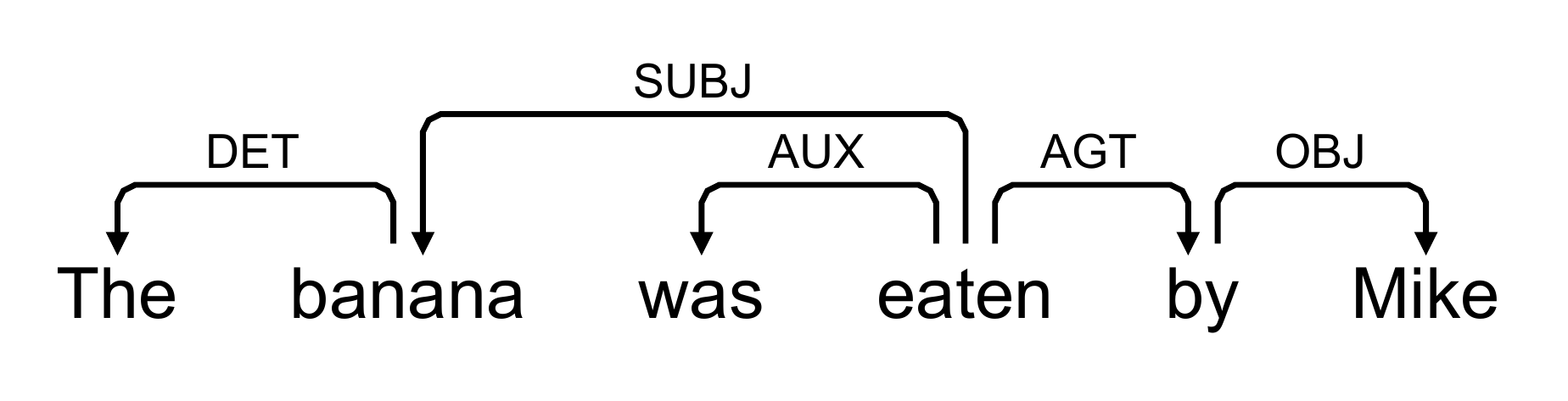}
\end{center}

In order to switch to the corresponding active form, two different linguistic transformations have to be defined for these two sentence; but, clearly, the first graph is a subgraph of the second one.
We don't want the transformation for the short passive on the left to apply on the long passive on the right.
we need to express a negative condition like ``there is no out edge labeled by AGT out of the main verb'' to prevent the unwanted transformation to occur.

\begin{wrapfigure}[4]{r}{80mm}
 \vspace{-7mm} \includegraphics[scale=.35]{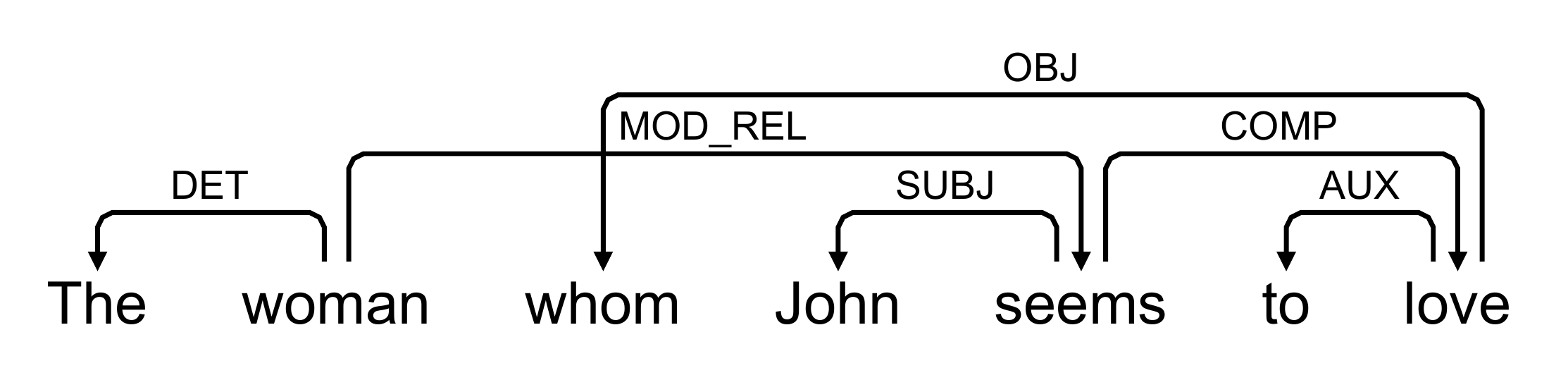}
\end{wrapfigure}
\paragraph{Long distance dependencies.}
\label{ex:ant}
Most of the linguistic transformation can be expressed with successive local transformation like the one above.
Nevertheless, there are some cases where more global rewriting is required; consider the sentence \txt{The women whom John seems to love}, for which we consider the syntactic structure on the right.
One of the steps in the semantic construction of this sentence requires to compute the antecedent of the relative pronoun \txt{whom} (the noun \txt{woman} in our example).

\begin{wrapfigure}{r}{70mm}
\vspace{-7mm}\begin{tabular}{c}
  \raisebox{-8mm}{\includegraphics[scale=.35]{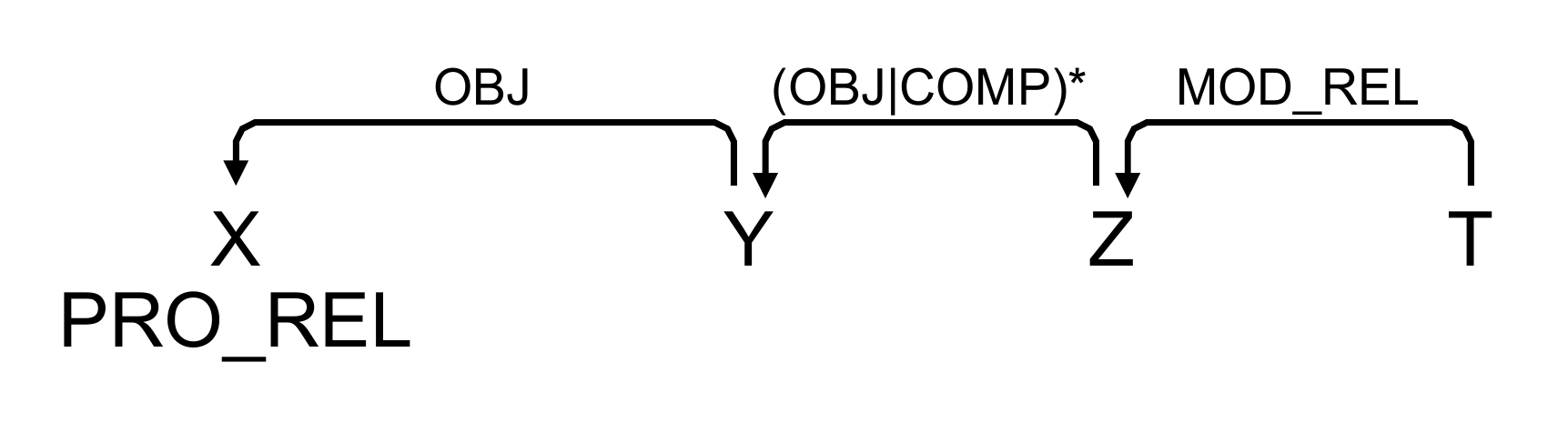}}\\ \add{X}{ANT}{T}
\end{tabular}
\end{wrapfigure}
The subgraph we have to search in our graph (which is depicted as a non-local pattern) and the graph modification to perform are given on the right.
The number of OBJ or COMP relations to consider (in the relation depicted as (OBJ$\mid$COMP)* in the figure) is unbounded (in linguistics, this phenomenon is called long distance dependencies); it is possible to construct grammatical sentences with an arbitrary large number of relations.

As we want to stay in the well-known framework of local rewriting, we will use several local transformations to implement such a non-local rule.

\begin{small}
  \begin{center}
    \begin{tabular}{l}
      \includegraphics[scale=.35]{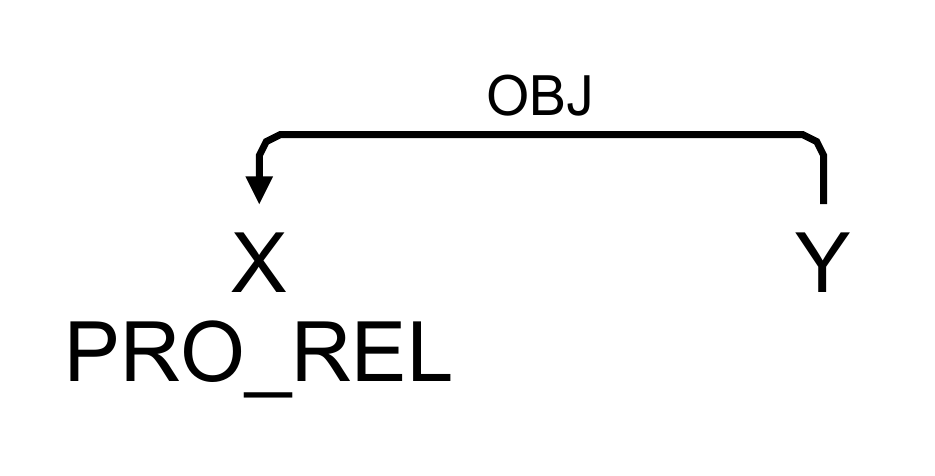} \\
      \add{X}{TMP}{Y} \\
    \end{tabular}
    \begin{tabular}{l}
      \includegraphics[scale=.35]{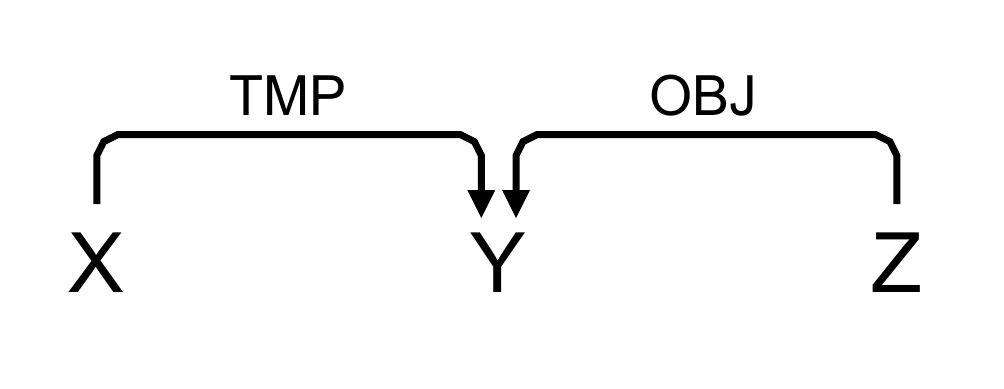} \\
      \del{X}{TMP}{Y} \\
      \add{X}{TMP}{Z} \\
    \end{tabular}
    \begin{tabular}{l}
      \includegraphics[scale=.35]{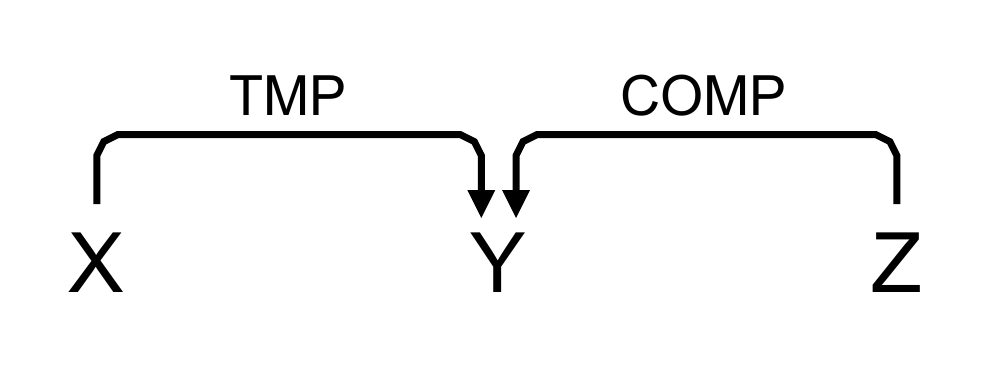} \\
      \del{X}{TMP}{Y} \\
      \add{X}{TMP}{Z} \\
    \end{tabular}
    \begin{tabular}{l}
      \includegraphics[scale=.35]{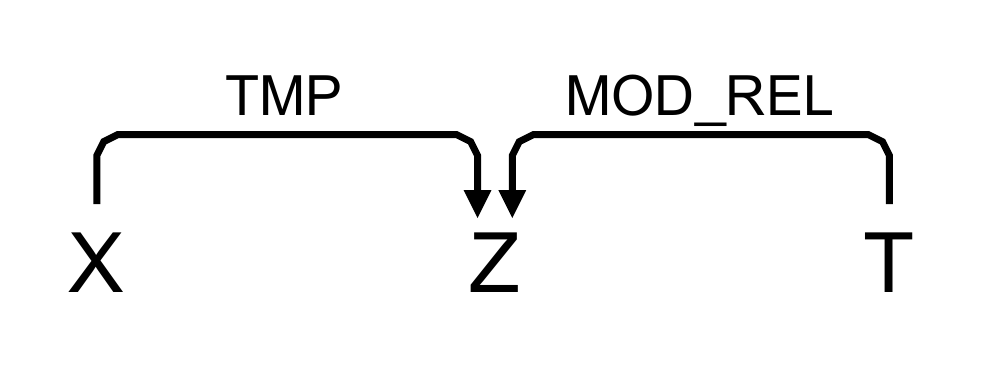} \\
      \del{X}{TMP}{Z} \\
      \add{X}{ANT}{T} \\
    \end{tabular}
  \end{center}
\end{small}
\medskip
The second and the third rules above preserve the set of nodes and the number of edges of each kind.
Hence, this kind of rule will require special treatment with respect to termination issues.

\section{Graph Rewriting for NLP}
\label{sec:def}

Before we enter into the technical sections, let us define some useful notations. First, we use the notation $\vec{c}$ to denote sequences. The empty sequence is written $\emptyset$. The length of a sequence is denoted by $\sizevec{c}$. We use the same notation for sets: the empty set is denoted $\emptyset$ and the cardinality of a set $S$ is written $\size{S}$. The context will make clear whether we are talking about sequences or sets.

Given a function $f: X \to Y$ and some sets $X' \subseteq X$ and $Y' \subseteq Y$, we define $f(X') \triangleq \{ f(x) \mid x \in X'\}$ and $f^{-1}(Y') \triangleq \{x \in X \mid f(x) \in Y'\}$;  the restriction of the function $f$ to the domain $X'$ is $\restr{f}{X'} : x' \in X' \mapsto f(x')$.
The function $\mathbf{c}_X : x \in X \mapsto c \in Y$ is the constant function on $X$.
The identity function is written $\Id$.
Finally, given a function $f : X \to Y$ and $(x,y) \in X \times Y$, the function $f[x \mapsto y]$ maps $t \neq x$ to $f(t)$ and $x$ to $y$.

The set of natural numbers is $\N$, integers are denoted by $\Z$.
Given two integers $a, b$, we define $[a,b] = \{ x \in \Z \mid a \leq x \leq b\}$.
\subsection{Graphs}

The graphs we consider are directed graphs with both labels on nodes and labels on edges.
We restrict the edge set: given some edge label $e$, there is at most one edge labeled $e$ between two given nodes $\alpha$ and $\beta$.
This restriction reflects the fact that, in NLP application, our edges are used to encode linguistic information which are relations.
We make no other explicit hypothesis on graphs: in particular, graphs may be disconnected, or have loops.

In this paper, we suppose given a finite set $\elabs$ of edge labels and another finite set $\nlabs$  of node labels.

\begin{definition}[Graph]
A \emph{graph} $G$ is defined as a triple $(\nodes, \ell, \edges)$ where
  \begin{itemize}
  \item $\nodes$ is a finite set of nodes;
  \item $\ell$ is a labeling function: $\ell : \nodes \mapsto \nlabs$;
  \item $\edges$ is a set of edges: $\edges \subseteq \nodes \times \elabs \times \nodes$.
  \end{itemize}
\end{definition}

Let $n, m \in \nodes$ and $e \in \elabs$.
When there is an edge from $n$ to $m$ labelled $e$ (\ie $(n,e,m) \in \edges$), we write $\arete{n}{e}{m}$ or $\arete{n}{}{m}$ if the edge label is not relevant.
If $G$ denotes some graph $(\nodes, \ell, \edges)$, then $\nodes_G$, $\ell_G$, $\edges_G$ denote respectively $\nodes, \ell$ and $\edges$.

\begin{definition}[Graph morphism]
A \emph{graph morphism} $\mu$ from the graph $G = (\nodes, \ell, \edges)$ to the graph $G' = (\nodes', \ell', \edges')$ is a function from $\nodes$ to $\nodes'$ such that:

\begin{itemize}
\item for all $n \in \nodes$, $\ell'(\mu(n)) = \ell(n)$;
\item for all $n, m \in \nodes$ and $e \in \elabs$, if $\arete{n}{e}{m} \in \edges$ then $\arete{\mu(n)}{e}{\mu(m)} \in \edges'$.
\end{itemize}

A graph morphism $\mu$ is said to be \emph{injective} if $\mu(n) = \mu(m)$ implies $n=m$. 
 We make the following abuse of notation: given some graph morphism $\mu:G \to G'$, and a set $E \subseteq \edges_G$, we let $\mu(E) = \{\arete{\mu(n)}{e}{\mu(m)} \mid \arete{n}{e}{m} \in E \}$.
\end{definition}


\begin{definition}[Basic pattern and basic matching]
  A \emph{basic pattern} $B$ is a graph.
  A \emph{basic matching} $\mu$ of the basic pattern $B$ in the graph $G$ is an \emph{injective} graph morphism $\mu$ (written $\mu: B \hookrightarrow G$).
\end{definition}

As shown in Section~\ref{sec:nlp},  negative conditions on patterns naturally arise in NLP.
We classify negative conditions in two categories: the local ones, that is negative conditions on edges within the basic pattern and non-local ones, that is negative conditions concerning edges between a node of the basic pattern and a node of the context (either in-edges or out-edges).

\begin{definition}[Pattern]
  A \emph{pattern} is a quadruple $P = (B, \noedges, \noin, \noout)$ of:
  \begin{itemize}
  \item a basic pattern $B = (\nodes_P, \ell_P, \edges_P)$;
  \item a set of forbidden edges $\noedges \subset  \nodes_P \times \elabs \times \nodes_P$ such that  $\noedges \cap \edges_P = \emptyset$;
  \item a set of forbidden in-edges $\noin \subset  \nodes_P \times \elabs$
  \item a set of forbidden out-edges $\noout \subset  \nodes_P \times \elabs$
  \end{itemize}
\end{definition}

Given a basic pattern $B$, we shorten $(B,\emptyset,\emptyset,\emptyset)$ to $(B,\vec{\emptyset})$.  In the following, given a pattern $P$, $\nodes_P$ and $\edges_P$ denote respectively the set of nodes of its basic pattern and the set of edges of its basic pattern.

\begin{definition}[Matching]
Let $P = (B, \noedges, \noin, \noout)$ be a pattern, $G = (\nodes, \ell, \edges)$ be a graph, and $\mu: B \hookrightarrow G$ be a basic matching.
We say that $\mu$ is a \emph{matching} from $P$ into $G$ (also written $\mu: P \hookrightarrow G$) whenever it satisfies the additional three conditions:
\begin{itemize}
\item $\mu(\noedges) \cap \edges = \emptyset$
\item for each $(n,e) \in \noin$, $\{p \in \nodes \setminus \mu(\nodes_P) \mid \arete{p}{e}{\mu(n)} \} = \emptyset$
\item for each $(n,e) \in \noout$, $\{p \in \nodes \setminus \mu(\nodes_P) \mid \arete{\mu(n)}{e}{p} \} = \emptyset$
\end{itemize}
\end{definition}

\begin{example}
\label{ex:matching}
Negative conditions are used to remove 'unwanted' matching. To see their effect, consider for instance the basic pattern $B_0$ and its two basic matchings $\mu_1$ and $\mu_2$ in $G_0$:

\begin{center}
\begin{tikzpicture}[->,transform shape,scale=0.7]]
\node[etat,fill=yellow] (B0) at (5.5,.5) {$b_0$:$\alpha$};
\node[etat,fill=yellow] (B1) at (7.5,.5) {$b_1$:$\beta$};
\path  (B0) edge node[auto] {A} (B1);

\path[left hook->] (4.3,.5) edge node[auto,above]{$\mu_1$}(2.7,.5);
\node[etat,fill=yellow] (Ga0) at (-2.5,0) {$g_0$:$\alpha$};
\node[etat,fill=yellow] (Ga1) at (-0.5,1) {$g_1$:$\beta$};
\node[etat] (Ga2) at (1.5,0) {$g_2$:$\alpha$};
  \path
  (Ga0) edge [bend left=20] node[above] {A} (Ga1)
         edge [bend right=20] node[below] {B} (Ga1)
         edge [bend right=20] node[below] {D} (Ga2)
  (Ga1) edge [bend left=20] node[above] {C} (Ga2)
  (Ga2) edge [bend left=20] node[below] {A} (Ga1)
  (Ga0) edge [loop left] node {E} (Ga0);

\path[right hook->] (8.7,.5) edge node[auto]{$\mu_2$}(10.3,.5);
\node[etat] (Gb0) at (12.5,0) {$g_0$:$\alpha$};
\node[etat,fill=yellow] (Gb1) at (14.5,1) {$g_1$:$\beta$};
\node[etat,fill=yellow] (Gb2) at (16.5,0) {$g_2$:$\alpha$};
  \path
  (Gb0) edge [bend left=20] node[above] {A} (Gb1)
         edge [bend right=20] node[below] {B} (Gb1)
         edge [bend right=20] node[below] {D} (Gb2)
  (Gb1) edge [bend left=20] node[above] {C} (Gb2)
  (Gb2) edge [bend left=20] node[below] {A} (Gb1)
  (Gb0) edge [loop left] node {E} (Gb0);
\end{tikzpicture}
\end{center}

\begin{itemize}
\item First, let $P_0 = (B_0, \vec{\emptyset})$. Then, $\mu_1$ and $\mu_2$ are (the) two matchings $P_0 \hookrightarrow G_0$.

\item Second, let the pattern $P_1 = (B_0, \{(b_1,C, b_0)\}, \emptyset, \emptyset)$;  then, $\mu_1$  is the only matching $P_1 \hookrightarrow G_0$.
\item Third, let the pattern $P_2 = (B_0, \emptyset, \{(b_0, D)\}, \{(b_0, D)\})$. Then, there is no matching of $P_2$ into $G_0$.
\end{itemize}

In the following, patterns $P_1$ and $P_2$ are depicted as:
\begin{center}
\begin{tikzpicture}[->,transform shape,scale=0.7]]
\node at (4.8,3) {$P_1 = $};
\node[etat] (P1_0) at (6,3) {$b_0$:$\alpha$};
\node[etat] (P1_1) at (8,3) {$b_1$:$\beta$};
\path  (P1_0) edge [bend left=20] node[auto] {A} (P1_1);
\path  (P1_1) edge [bend left=20] node[auto] {C}  node[sloped, midway] {\large $\times$} (P1_0);
\end{tikzpicture}
\qquad \qquad
\begin{tikzpicture}[->,transform shape,scale=0.7]]
\node at (3.8,3) {$P_2 = $};
\node[etat] (P2_0) at (6,3) {$b_0$:$\alpha$};
\node[etat] (P2_1) at (8,3) {$b_1$:$\beta$};
\path  (P2_0) edge node[auto] {A} (P2_1);
\path (P2_0) edge node[auto, above, pos=.3] {D} node[sloped, pos=0.7] {\large $\times$} (4.5,3.6);
\path (4.5,2.4) edge node[auto, below, pos=.7] {D} node[sloped, pos=0.3] {\large $\times$} (P2_0);
\end{tikzpicture}
\end{center}
\end{example}

\subsection{Graph decomposition}
\label{ssec:graph_decomp}
The proper description of actions of a rule on some graph $G$ requires first the definition of two partitions: one on nodes and the other on edges. They are both induced by the matching of some pattern $P$ into $G$.

\begin{definition}[Nodes decomposition: pattern image, crown and context]\label{def:node_decomp}
  Let $\mu: P \hookrightarrow G$ a matching from the pattern $P$ into the graph $G = (\nodes, \ell, \edges)$.
  Nodes of $G$ can be split in a partition of three sets $\nodes = \pattim_\mu \oplus \crown_\mu \oplus \context_\mu$:
  \begin{itemize}
  \item the \emph{pattern image} is $\pattim_{\mu} = \mu(\nodes_P)$;
  \item the \emph{crown} contains nodes outside the pattern image which are directly connected to the pattern image:
    $\crown_{\mu} = \{n \in \nodes \setminus \pattim_\mu \mid \exists p\in\pattim_\mu \mathrm{~such~that~}\arete{n}{}{p} \mathrm{~or~} \arete{p}{}{n}\}$;
    \item the \emph{context} contains nodes not linked to the pattern image: $\context_{\mu} = \nodes \setminus (\pattim_{\mu} \cup \crown_{\mu})$.
  \end{itemize}
\end{definition}

\begin{definition}[Edges decomposition: pattern edges, crown edges, context edges and pattern-glued edges]\label{def:edge_decomp}
  Let $\mu: P \hookrightarrow G$ a matching from the pattern $P$ into the graph $G = (\nodes, \ell, \edges)$.
  Edges of $G$ can be split in a partition of four sets $\edges = \mu(\edges_P) \oplus \crown^\mu \oplus \context^\mu \oplus \glued^\mu$:
  \begin{itemize}
  \item the \emph{pattern edges} is $\mu(\edges_P)$;
  \item the \emph{crown edges} set contains edges which links a pattern image node to a crown node:
    $\crown^{\mu} = \{\arete{n}{}{m} \in \edges \mid n\in \pattim_\mu \wedge m \in \crown_\mu\} \cup  \{\arete{n}{}{m} \in \edges \mid n\in \crown_\mu \wedge m \in \pattim_\mu\}$;
    \item the \emph{context edges} set contains edges which connect two nodes that are not in the pattern image:
      $\context^{\mu} =  \{\arete{n}{}{m} \in \edges \mid n\notin \pattim_\mu \wedge m \notin \pattim_\mu\}$.
    \item the \emph{pattern-glued edges} set contains edges which are not pattern edges but which connect two nodes that are in the pattern image:
      $\glued^{\mu} =  \{\arete{n}{}{m} \in \edges \mid n\in \pattim_\mu \wedge m \in \pattim_\mu\} \setminus \mu(\edges_P) $.
  \end{itemize}
\end{definition}

\subsection{Rules}
In our graph rewriting framework, the transformation of the graph is described through some atomic commands (like in~\cite{echahed}). Commands definition refer to some pattern $P$ and pattern nodes $\nodes_P$ are used as identifiers.
Let $a,b \in \nodes_P$, $\alpha \in \nlabs$ and $e \in \elabs$, the five kinds of commands are $\ren{a}{\alpha}$, $\del{a}{e}{b}$, $\add{a}{e}{b}$, $\delnode{a}$ and $\shift{a}{b}$.

Their names speak for themselves, however, we will come back to their precise meaning in the subsection below.
Before this, to ensure that commands always refer to valid node identifiers, we restrict command sequences to \emph{consistent} sequences, that is sequences $c_1, \ldots, c_k$  such that for each command $c_i$, $1 \leq i \leq k$, which is a node deletion command $\delnode{a}$ for some $a \in \nodes_P$, then the node name $a$ does not occur in any command $c_j$ with $i < j \leq k$.

\begin{definition}[Rule]
A \emph{rule} $R$ is a pair $R = \rule{P}{\vec{c}}$ of a pattern $P$ and a sequence of commands $\vec{c}$ consistent with respect to $P$.
A rule $R = \rule{P}{\vec{c}}$ is said to be \emph{node-preserving} if $\vec{c}$ does not contain any $\xdelnode$ command.
\end{definition}

\subsection{Graph Rewrite System}

Let $G = (\nodes, \ell, \edges)$ a graph, $R = \rule{P}{\vec{c}}$ a rule and $\mu: P \hookrightarrow G$ a matching.
The application of the sequence $\vec{c}$ on $G$ is a new graph which is written $G \cdot_\mu \vec{c}$ (shortened $G \cdot \vec{c}$ when $\mu$ is clear from the context) and is defined by induction on the length $k$ of $\vec{c}$.
If $k=0$, $G \cdot \emptyset = G$.
If $k>0$, let $G' = (\nodes', \ell', \edges')$ be the graph obtained by application of the sequence $c_1, \ldots, c_{k-1}$; then we consider each command in turn:
\begin{description}
\item [Label:] The command $c_k = \ren{a}{\alpha}$ changes the label of the node $\mu(a)$: $G \cdot \vec{c} = (\nodes', \ell'', \edges')$ with $\ell''=\ell'[\mu(a) \mapsto \alpha] $.
\item [Delete:] The command $c_k = \del{a}{e}{b}$ deletes the edge from $\mu(a)$ to $\mu(b)$ labelled with $e \in \elabs$; $G \cdot \vec{c} = (\nodes', \ell', \edges'')$ with $\edges''=\edges' \setminus \{\arete{\mu(a)}{e}{\mu(b)}\}$.
\item [Add:] The command $c_k = \add{a}{e}{b}$ adds an edge from $\mu(a)$ to $\mu(b)$ labelled with $e \in \elabs$; $G \cdot \vec{c} = (\nodes', \ell', \edges'')$ with $\edges''=\edges' \cup \{\arete{\mu(a)}{e}{\mu(b)}\}$.
\item[Delete node:] The command $c_k = \delnode{a}$ removes the node $\mu(a)$ of $G'$; $G \cdot \vec{c}= (\nodes'', \ell'', \edges'')$ with $\nodes'' = \nodes' \setminus \{\mu(a)\}$, $\ell'' = \restr{\ell'}{\nodes''}$ and $\edges'' = \edges' \cap \{\nodes'' \times \elabs \times \nodes''\}$.
\item[Shift edges:] The command $c_k = \shift{a}{b}$ changes in-edges of $\mu(a)$ starting from the crown to in-edges of $\mu(b)$ and all out-edges of $\mu(a)$ going to the crown to out-edges of $\mu(b)$.
Formally, $G \cdot \vec{c} = (\nodes', \ell', \edges'')$ with the set $\edges''$ defined by, for all $e\in\elabs$:
    \begin{itemize}
    \item for all $p \in \crown_\mu$, $\arete{\mu(b)}{e}{p} \in \edges''$ iff $\arete{\mu(b)}{e}{p} \in \edges'$ or $\arete{\mu(a)}{e}{p} \in \edges'$;
    \item for all $p \in \crown_\mu$, $\arete{p}{e}{\mu(b)} \in \edges''$ iff $\arete{p}{e}{\mu(b)} \in \edges'$ or $\arete{p}{e}{\mu(a)} \in \edges'$;
    \item for all $p,q\in\pattim_\mu$, $\arete{p}{e}{q} \in \edges''$ iff $\arete{p}{e}{q} \in \edges'$;
    \item for all $p,q\in\crown_\mu \cup \context_\mu $, $\arete{p}{e}{q} \in \edges''$ iff $\arete{p}{e}{q} \in \edges'$.
    \end{itemize}
\end{description}

The commands $\xren$, $\xdel$ and $\xadd$ are called local commands: they modify only the edges and the nodes described in the pattern.
The commands $\xdelnode$ and $\xshift$ are non-local: they can modify edges outside the pattern.
Note that a rule $\xadd$ (resp. $\xdel$) may have no effect if the edge already exists (resp. does not exist).
Note also that we can suppose that for a given sequence $\vec{c}$ and a given triple $(a,e,b)$, there is at most one rule $\del{a}{e}{b}$ or $\add{a}{e}{b}$ in $\vec{c}$ (if not, only the last one is effective).
Hence, we can define uniform rules:

\begin{definition}[Uniform rule]
For $\vec{c} = c_1, \ldots, c_k$ without node deletion, the rule $\rule{P}{\vec{c}}$ is \emph{uniform} iff for all $1 \leq i \leq k$, if $c_i = \add{n}{e}{m}$ then $(n,e,m) \in \noedges_P$ and if $c_i =  \del{n}{e}{m}$ then $(n,e,m) \in \edges_P$.
\end{definition}

\begin{definition}[Rewrite step]
Let $G = (\nodes, \ell, \edges)$ a graph, $R = \rule{P}{\vec{c}}$ a rule and $\mu: P \hookrightarrow G$ a matching.
Let $G' = G \cdot \vec c$, then we say that $G$ rewrites to $G'$ with respect to the rule $R$ and  the matching $\mu$.
We write it $G \to_{R,\mu} G'$ or $G \to_R G'$ or even simply $G \to G'$.
\end{definition}

\begin{definition}[Graph Rewrite System]
A \emph{Graph Rewrite System} $\prog$ is a finite set of rules.
\end{definition}


In our application, the translation of the syntax to semantics is split into several independent levels of transformation driven by linguistic consideration (such as translation of passive forms to active ones, computation of the deep subject of infinites).
Rules are then grouped in subsets called \emph{modules} and modules apply sequentially; each module being used as a graph rewrite system on the outputs of the previous module.

\begin{lemma}[Linear modification]
\label{le:modifs}
Given a GRS $\prog$, there is a constant $C>0$ such that, for any
rewriting step $G \to_{R,\mu} G'$ the two canonical corresponding edge decompositions $\edges_G = \context^\mu \oplus \contextcomp^\mu$ and $\edges_{G'} = \context^\mu \oplus \contextcomp'^\mu$ satisfy:
$$ \size{\contextcomp^\mu} \le C \times (\size{G}+1) \mbox{~and~} \size{\contextcomp'^\mu} \le C \times (\size{G}+1)$$
\end{lemma}

\begin{proof}
Let $C = \max \{2 \times \size{P}^2 \times \size{\elabs}\} \mid \rule{P}{\vec{c}} \in \prog\}$.
Both in $G$ and $G'$, edges that are not in the context are either between two pattern nodes or between a pattern node and a crown node.
The total number of edges of the first kind (either pattern edges or glued-pattern edges) is bounded by $\size{P}^2 \times \size{\elabs}$.
For each pattern node, the number of edges which connect this node to some non-pattern node is bounded by $2 \times \size{G} \times \size{\elabs}$ and so the total number of edges which link some pattern node to some non-pattern node is bounded by $2 \times \size{G} \times \size{\elabs}\times \size{P}$. Putting everything together, $\size{\contextcomp^\mu} \le C \times (\size{G}+1)$ and $\size{\contextcomp'^\mu} \le C \times (\size{G}+1)$.
\end{proof}

\section{Weighted GRS}
\label{sec:term}

We recall that a GRS is said to be (strongly) terminating whenever there is no infinite sequence $G_1 \to G_2 \to \cdots$.
Given a terminating GRS $\prog$ and a graph $G$, we define the derivation height of $G$, next denoted $\dhe_\prog(G)$,  to be the length of the longest derivation starting from $G$ if such a derivation exists.
If $\dhe_\prog(G)$ is defined for all $G$ such that $|G| \leq n$, then we define the derivation height of $\prog$ by: $\dhe_\prog(n) = \max \{ \dhe_\prog(G) \mid \size{G} \leq n\}$.

Actually, for non-size increasing GRS as presented above, we have immediately the decidability of non-uniform termination. That is, given some GRS $\prog$ and some graph $G$, one may decide whether there is an infinite sequence $G_1 \to G_2 \to G_3 \to \cdots$. Indeed, one may observe that for such sequence, for all $i \in \N$, $\size{G_i} \leq \size{G}$. Thus, the $G_i$'s range in the finite set $\bigg_{\leq \size{G}}$ of graphs of size less or equal to $\size{G}$. Consequently, either the system terminates or there is some $j \leq |\bigg_{\leq \size{G}}|$  and some $k \leq j$ such that $G_j = G_k$. To conclude, to decide non-uniform termination, it is sufficient to compute all the (finitely many) possibilities of rewriting $G$ in less than $|\bigg_{\leq \size{G}}|$ steps and to verify the existence of such a $j$ and $k$ above. Finally, since $|\bigg_{\leq \size{G}}| \leq 2^{O(|G|^2)}$, the procedure as described above takes exponential time.

However, uniform termination--- given a GRS, is it terminating?--- of non-size increasing GRS remains an open problem. Uniform termination was proved undecidable when we drop the property of non-size increasingness (cf. Plump~\cite{Plump98a}).  As a consequence, there is a need to define some termination method pertaining to non-size increasing GRS. Compared to standard work in termination~\cite{Plump:1995:TGR:647676.731819,DBLP:conf/gg/GodardMMS02}, there are two difficulties: first, our graphs may be cyclic, thus forbidding methods developed for DAGs such as term-graphs. Second, using term rewriting terminology, our method should operate for some non-simplifying GRS, that is GRS for which the output may be "bigger" than the input. Indeed, the NLP programmer sometimes wants to compute some \emph{new} relations, so that the input graph is a strict sub-graph of the resulting graph.

\subsection{Termination by weight analysis}

In the context of term-rewriting systems, the use of weights is very common to prove termination. There are many examples of such orderings, Knuth-Bendix Ordering~\cite{KB70} to cite one of them.
We recall that all graphs we consider are defined relatively to two signatures $\elabs$ of edge labels and $\nlabs$ of node labels.

\begin{definition}[Edge weight, node weight]
An \emph{edge weight} is a mapping $w: \elabs \to \Z$.
Given some subset $E$ of edges of $G$, the weight of $E$ is $w(E) = \sum_{\arete{n}{e}{m} \in E} w(e)$.
The edge weight of a graph $G$ is $w(G) = w(\edges_G)$.
A \emph{node weight} is a mapping $\eta : \nlabs \to \Z$. For a graph $G = (\nodes_G,\ell_G,\edges_G)$, we define $\eta(G) = \sum_{n \in \nodes_G} \eta(\ell_G(n))$.
\end{definition}

Let us make some observations.
Let $\size{G}_e$ denote the number of edges in $G$ which have the label $e$, then $w(G) = \sum_{e\in\elabs} w(e) \times \size{G}_e$. Second, for a pattern matching $\mu: P \hookrightarrow G$, $w(\mu(P))= w(P)$.

The weight of a graph may be negative.
This is not standard, but it is useful here to cope with non-simplifying rules, that is rules which add new edges.
Since a graph $G$ has at most $\size{\elabs} \times \size{G}^2$ edges, the following lemma is immediate.

\begin{lemma}\label{le:weight}
Given an edge weight $w$ and a node weight $\eta$,  let $K_w = \max_{e \in \elabs}(|w(e)|)$, $K_E = \size{\elabs} \times K_w$, $K_\eta = \max_{\alpha \in \nlabs}(|\eta(\alpha)|)$, then
\begin{enumerate}[(a)]
\item\label{le:weight:wE} for each subset of edges $E \subset \edges_G$ of some graph $G$, we have $w(E) \le K_w \times \size{E}$.
\item\label{le:weight:wG} for each graph $G$, we have $-K_E \times \size{G}^2 \leq w(G) \leq K_E \times \size{G}^2$;
\item\label{le:weight:eta} for each graph $G$, we have $|\eta(G)| \leq K_\eta \times |G|$.
\end{enumerate}
\end{lemma}

\begin{definition}
Let $R=\rule{P}{\vec{c}}$ a rule, we define inductively $\shiftf_{\vec{c}} : \nodes_P \to \nodes_P$ which describes the global effect of the $\xshift$ commands in a rule:
$\shiftf_\emptyset = \Id$; $\shiftf_{\vec{c},\shift{m}{n}} = \Id[m \mapsto n] \circ \shiftf_{\vec{c}}$ and $\shiftf_{\vec{c},c} = \shiftf_{\vec{c}}$ if $c$ is not a $\xshift$ command.
\end{definition}

\begin{definition}[Compatible weight]\label{def:comp-weight}
Given a rule $R=\rule{(P,\noedges, \noin, \noout )}{\vec{c}}$, an edge weight $w$ is said to be \emph{compatible} with $R$ if: 

\begin{enumerate}
\item either $\vec{c}$ contains a $\xdelnode$ command
\item or $R$ is a node-preserving rule and satisfy the three properties:
  \begin{enumerate}
  \item $R$ is uniform,
  \item $w(P\cdot_\Id \vec{c}\,) < w(P)$,
  \item for all $e \in \elabs$ such that $w(e) < 0$, for all $n \in \shiftf(\nodes_P)$, let $M_n \subset \edges_P$ be the set $\shiftf_{\vec c}^{-1}(n)$; then  $M_n$ contains at most one element $m$ such that $(m,e) \not\in \noin$ and $M_n$ contains at most one element $m' \in M_n$ such that $(m',e) \not\in \noout$.
  \end{enumerate}
\end{enumerate}

An edge weight is said to be compatible with a GRS $\prog$ if it is compatible with all its rules.
A weighted GRS is a pair $(\prog,w)$ of a GRS and a compatible weight.
\end{definition}

Hypothesis (2.b) will serve to manage edges in the pattern images while Hypothesis (2.c) will serve for the crown edges.
One may note that when there is no $\xshift$ commands in the rule, the Hypothesis (2.c) holds whatever $w$ is.
Indeed, in that case, $\shiftf$ is the identity function and all the sets $M_n$ are singletons.

\begin{lemma}\label{le:comp-weight}
Let $(\prog,w)$ a weighted GRS, let $G \to G'$ be a rewrite step of $\prog$.
Either $\size{G} > \size{G'}$ or $\size{G} = \size{G'}$ and $w(G) > w(G')$.
\end{lemma}

The problem of the synthesis is the following.
Given a GRS $\prog$, is there a weight $w$ compatible with $\prog$? 
Since the existence of weights can be described in Presburger's arithmetic, we have a positive answer:

\begin{theorem}\label{th:decidable}
Given a GRS $\prog$, one may decide whether or not it has a compatible weight.
\end{theorem}

Second point, the existence of weights induce termination:

\begin{theorem}\label{th:comp-weight}
Any weighted GRS $(\prog,w)$ is strongly terminating in quadratic time. Moreover, this quadratic bound is a lower bound: there is a GRS $\prog$ with a compatible weight such that $\dhe_\prog(n)\geq O(n^2)$.
\end{theorem}

Condition~(2.c) of Definition~\ref{def:comp-weight} is necessary.
Here is a counter-example of a non-terminating system with a compatible weight up to this condition.
Consider the two rules $\rule{Q_1}{\vec{c_1}}$ and $\rule{Q_2}{\vec{c_2}}$:
\begin{center}
\begin{tikzpicture}
\tikzstyle{etat}=[circle,draw=black]
\node[etat] (Q) at (0,0) {0:e};
\node[etat] (P) at (2,0) {1:e};
\draw[->] (Q) to[bend left] node[auto]{A} (P) ;
\draw[->] (P) to[bend left] node[auto]{B}(Q);
\node [anchor=west] at (2.5, 0.25) {$\del{0}{A}{1}$};
\node [anchor=west] at (2.5, -0.25) {$\shift{0}{1}$};

\node[etat] (R) at (10,0){3:e};
\node[etat] (S) at (12,0){4:e};
\node[etat] (T) at (8,0){5:e};
\draw[->](T) to[bend left] node[auto]{C} (S);
\draw[->](S) to node[auto]{B} (R);
\node [anchor=west] at (12.5, 0.25) {$\add{3}{A}{4}$};
\node [anchor=west] at (12.5, -0.25) {$\add{5}{C}{3}$};
\end{tikzpicture}
\end{center}
Set $w(A) = w(B) = 1$ and $w(C) = -2$.
Observe that $w(Q_1 \cdot \vec{c_1}\,) = 1 <  2 = w(Q_1)$ and $w(Q_2 \cdot \vec{c_2}) = -2 < -1 = w(Q_2)$.
However, there is an infinite sequence $G_1 \to_{R_1} G_2 \to_{R_2} G_1 \to_{R_1}  \cdots$ with $G_1$ and $G_2$ being:

\begin{center}
 \begin{tikzpicture}
\tikzstyle{etat}=[circle,draw=black]
\node[etat] (P) at (0,0) {0:e};
\node[etat] (Q) at (2,0) {1:e};
\node[etat] (R) at (4,0) {2:e};
\draw[->] (Q) to node[auto]{A}(R);
\draw[->] (R) to[bend left] node[auto]{B}(Q);
\draw[->] (P) to[bend left] node[auto]{C} (R);
\draw[->] (P) to node[auto]{C} (Q);

\node[etat] (S) at (7,0) {0:e};
\node[etat] (T) at (9,0) {1:e};
\node[etat] (U) at (11,0) {2:e};
\draw[->] (U) to[bend left] node[auto]{B}(T);
\draw[->] (S) to[bend left] node[auto]{C} (U);

\node at (5.5,0) {$\leftrightarrows$};
\node at (-1,0) {$G_1=$};
\node at (12,0){$=G_2$};
\end{tikzpicture}
\end{center}

\begin{proof}[Proof sketch of Theorem~\ref{th:comp-weight}]
We begin to show the lower bound.
Let $\elabs=\{E\}$,  $\nlabs = \{ e\}$.
Consider the two rules GRS $\prog$ defined by the two basic patterns:
\begin{center}
\begin{tikzpicture}
\tikzstyle{etat}=[circle,draw=black]
\node[etat] (P) at (0,0) {0:e};
\node[etat] (Q) at (2,0) {1:e};
\draw[->] (P) to node[auto]{E} (Q);
\node [anchor=west] at (2.5, 0) {$\del{0}{E}{1}$};

\node[etat] (S) at (8,0) {2:e};
\draw [->] (S) to[loop left] node[auto] {E} (S);
\node [anchor=west] at (9,0) {$\del{2}{E}{2}$};
\end{tikzpicture}
\end{center}

Set $w(E) = 1$.
The rules are compatible with $w$.
Each rule deleting exactly one edge, since the clique $C_n$ of size $n$ has $n^2$ edges, the derivation height $\dhe_\prog(C_n) = n^2$.
The lower bound follows.

For the upper bound, let $C$ be the constant as defined by Lemma~\ref{le:modifs}, let $K = \max(1,K_w)$ (we recall that $K_w= \max_{e \in\elabs}(|w(e)|)$).
Finally, let $H = \max \{ n \mid (P,c_1, \ldots, c_n) \in \prog\}$.
Let $A = 2 \times K\times C \times (H+1)+1$.
Let $\Omega$ be the 'energy function' defined on graphs $\Omega(G) = w(G) +  A \times \size{G}^2$.
For each rule application $G \to G'$, one may verify that  $\Omega(G)>\Omega(G')$.
The last inequality together with Lemma~\ref{le:weight} leads to the conclusion.
\end{proof}

Full proofs of Theorem~\ref{th:decidable} and of Theorem~\ref{th:comp-weight} are given in~\cite{MSCS2013}.

\subsection{Termination by lexicographic weight}

In our experiments, in most cases, the weight analysis of the preceding section was sufficient.
The main counter-example is however systems composed of rules as given in Section~\ref{ex:ant}.
The GRS is strongly terminating but there is no compatible weight.
This section provides a conciliable extension of this termination proof method.
With a little abstraction, the linguistic example of Section~\ref{ex:ant} about long distance dependencies is computed by some 'non-local rule' $R_{nl}$:
\begin{center}
\begin{tikzpicture}[->,transform shape,scale=0.7]]
\node at (-1,-1) {\Large $R_{nl} = $};
\node[etat,fill=yellow] (b0) at (2,-2) {$b_0$:$P$};
\node[etat,fill=yellow] (b1) at (2,0) {$b_1$:$X$};
\node[etat,fill=yellow] (b2) at (4,0) {$b_2$:$X$};
\path (b1) edge node[auto] {O} (b0);
\path (b2) edge node[auto, above] {O} (b1);
\path (b0) edge node[auto, pos=.3] {A} node[sloped, pos=0.7] {\large $\times$} ($ (b0) - (2,0) $);

\path (5.4,0) edge [dotted,-] (6.6,0);

\node[etat,fill=yellow] (bp0) at (8,0) {$b'_0$:$X$};
\node[etat,fill=yellow] (bp1) at (10,0) {$b'_1$:$X$};
\path  (bp1) edge node[auto, above] {O} (bp0);
\node[etat,fill=yellow] (bp2) at (10,-2) {$b'_2$:$X$};
\path  (bp2) edge node[auto] {M} (bp1);

\path  (b0) edge node[auto,above=3pt] {A}  node[sloped, midway] {\large $\times$} (bp2);

\node at (13,-1){\large $\add{b_0}{A}{b'_2}$};
\end{tikzpicture}
\end{center}

\noindent Such non-local rules can be implemented by rules:

\begin{figure}[h]
\begin{center}
  \begin{tabular}{|p{3cm}|p{3.3cm}|p{4.0cm}|p{3cm}|}
    \hline
    {\sc Init} & {\sc Rec} & {\sc Stop} & {\sc Clean}
    \\ \hline
    \begin{tikzpicture}[->,transform shape,scale=0.7]]
      \node[etat] (b0) at (2,-2) {$b_0$:$P$};
      \node[etat] (b1) at (2,0) {$b_1$:$X$};
      \path (b1) edge node[auto] {O} (b0);
      \path (b0) edge node[auto, pos=.3] {A} node[sloped, pos=0.7] {\large $\times$} ($ (b0) + (1.7,0) $);
      \path ($ (b1) + (1.7,0) $) edge node[auto, below, pos=.7] {E} node[sloped, pos=0.3] {\large $\times$} (b1);
     \node[right] at (0,-3.4) {\large $\ren{b_0}{P_\diamond}$};
      \node[right] at (0,-4) {\large $\add{b_0}{E}{b_1}$};

      \path ($ (b0) + (-1.4,0.8) $) edge node[auto, below, pos=.7, above] {E} node[sloped, pos=0.3] {\large $\times$} (b0);
      \path (b0) edge node[auto, below, pos=.3] {E} node[sloped, pos=0.7] {\large $\times$} ($ (b0) + (-1.4,-0.8) $);

    \end{tikzpicture}
    &
    \begin{tikzpicture}[->,transform shape,scale=0.7]]
      \node[etat,minimum size=1.2cm] (b0) at (2,-2) {$b_0$:$P_\diamond$};
      \node[etat] (b1) at (2,0) {$b_1$:$X$};
      \node[etat] (b2) at (4,0) {$b_2$:$X$};
      \path (b0) edge  node[auto] {E} (b1);
      \path (b2) edge  node[auto, above] {O} (b1);
      \path (b0) edge node[auto, pos=.3] {A} node[sloped, pos=0.7] {\large $\times$} ($ (b0) + (1.7,0) $);
      \path ($ (b2) + (1.7,0) $) edge node[auto, below, pos=.7] {E} node[sloped, pos=0.3] {\large $\times$} (b2);
      \node[right] at (1,-3) {\large $\add{b_0}{E}{b_2}$};
    \end{tikzpicture}
    &
    \begin{tikzpicture}[->,transform shape,scale=0.7]]
      \node[etat,minimum size=1.2cm] (b0) at (2,-2) {$b_0$:$P_\diamond$};
      \node[etat] (b1) at (2,0) {$b_1$:$X$};
      \node[etat] (b2) at (4,0) {$b_2$:$X$};
      \path (b0) edge  node[auto] {E} (b1);
      \path (b2) edge  node[auto, above] {M} (b1);
      \path (b0) edge node[auto, pos=.3] {A} node[sloped, pos=0.7] {\large $\times$} ($ (b0) + (1.7,0) $);
      \path ($ (b2) + (1.7,0) $) edge node[auto, below, pos=.7] {E} node[sloped, pos=0.3] {\large $\times$} (b2);
      \node[right] at (0,-3.4) {\large $\add{b_0}{A}{b_2}$};
      \node[right] at (0,-4) {\large $\ren{b_0}{P}$};

      \path ($ (b0) + (-1.4,0.8) $) edge node[auto, below, pos=.7, above] {E} node[sloped, pos=0.3] {\large $\times$} (b0);
      \path (b0) edge node[auto, below, pos=.3] {E} node[sloped, pos=0.7] {\large $\times$} ($ (b0) + (-1.4,-0.8) $);
    \end{tikzpicture}
&
    \begin{tikzpicture}[->,transform shape,scale=0.7]]
      \node[etat,minimum size=1.2cm] (b0) at (2,-2) {$b_0$:$P$};
      \node[etat] (b1) at (2,0) {$b_1$:$X$};
      \path (b0) edge  node[auto] {E} (b1);
      \node[right] at (1,-3) {\large $\del{b_0}{E}{b_1}$};
    \end{tikzpicture}\\\hline 
  \end{tabular}
  \caption{Local implementation of the non-local rule}
  \label{fig:local_rules}
  \end{center}
\end{figure}

\noindent However, these rules are not compatible with any weight. Actually, as justified in~\cite{MSCS2013}, there is no implementation of such a rule by some weighted rules.

\newcommand{\lex}{{\text{lex}}}
Given an order $\prec$ on some set $U$, its lexicographic extension to sequences in $U$ is
defined by $(u_1, \ldots, u_k) \prec_\lex (v_1, \ldots, v_m)$ iff
$\exists j \leq \min(m,k) : u_j \prec v_j \wedge \forall i < j : u_i = v_i$.
The order $\prec_\lex$ is not well-founded in general, but its restriction to sequences of equal length is such as soon as $\prec$ is well-founded.

\begin{definition}[Contextual weight]
An \emph{edge contextual weight} is a (finite) map $\omega: \nlabs \times \elabs \times \nlabs \to \Z$. As for weights, it extends to any set $E \subseteq \edges_G$ of some graph $G$ by: $\omega(E) = \sum_{\arete{n}{e}{m} \in E} \omega(\ell(n),e,\ell(m))$. And the weight of a graph is $\omega(G) = \omega(\edges_G)$.

A \emph{contextual weight} is a 4-tuple $\pi=(a,\omega,b,\eta)$ with $a, b \in \N$, $\omega$ an edge contextual weight and $\eta$ a node weight.
We define $\pi(G) = a \times \omega(G) + b \times \eta(G)$.

Let $e \in \elabs$, if $a\neq 0$ and there are $\alpha, \beta,\alpha', \beta' \in \nlabs$ such that $\omega(\alpha, e, \beta) \not= \omega(\alpha', e, \beta')$, then we say that $\pi$  is \emph{$e$-fragile}.
\end{definition}

\begin{definition}\label{def:comp-context-weight}
Given an edge weight $w_0 : \elabs \to \Z$, given $k$ contextual  weights $\pi_1, \ldots, \pi_k$ and a rule $R = \rule{P}{\vec{c}}$, we write $P' = P\cdot_\Id \vec{c}$. We say that $R$ is \emph{compatible} with $(w_0, \pi_1, \ldots, \pi_k)$ iff:
\begin{enumerate}
\item either $\vec{c}$ contains a $\xdelnode$ command,
\item or $R$ is an uniform and node-preserving rule such that:
  \begin{enumerate}
  \item either the two properties below hold
    \begin{enumerate}[(i)]
    \item $w_0(P') < w_0(P)$;
    \item and for all $e\in\elabs$ such that $w(e) < 0$, for all $n \in \shiftf(\nodes_P)$, let $M_n$ be the set $\shiftf^{-1}(n)$; then $M_n$ contains at most one element $m$ such that  $(m,e) \not\in \noin$ and $M_n$ contains at most one element $m' \in M_n$ such that $(q,m') \not\in \noout$.
    \end{enumerate}
  \item or the four properties below hold
    \begin{enumerate}[(i)]
    \item $w_0(P') = w_0(P)$;
    \item $(\pi_1(P'),\ldots,\pi_k(P')) <_\lex  (\pi_1(P),\ldots,\pi_k(P))$;
    \item if $\vec{c}$ contains a command $\ren{n}{\alpha}$ and if some $\pi_i$ is $e$-fragile, then $(n,e) \in \noin \cup \noout$;
    \item $\vec{c}$ does not contain any $\xshift$ commands.
    \end{enumerate}
  \end{enumerate}
\end{enumerate}		
\end{definition}

When a weight $w_0$ and $k$ contextual weights are compatible with all the rules of some GRS $\prog$, we say that $\prog$ is lexicographically weighted by $(w_0,\pi_1,\ldots, \pi_k)$.

\begin{example}
We define $w_0 = \mathbf{0}_\elabs [A \mapsto -1]$, and $\omega = \mathbf{0}_{\nlabs \times \elabs \times \nlabs}[(P,E,X) \mapsto 1, (P_\diamond,E,X) \mapsto -1]$.
Consider the lexicographic weight $\pi = (1,\omega,0,\mathbf{0}_{\nlabs})$.
For rules in Figure~\ref{fig:local_rules}, we have: rule {\sc Stop} decreases by (2.a); rules {\sc Init} and {\sc Rec} decrease by (2.b): there is one more edge labeled $E$ starting from $P_\diamond$ and rule {\sc Clean} decreases by (2.b): one edge labeled $E$ starting from $P$ disappears.
\end{example}

\begin{theorem}
Whenever a program $\prog$ is compatible with the lexicographic weight $(w_0,\pi_1, \ldots, \pi_k)$, it is strongly terminating in polynomial time.
The bound is tight, that is for all $k>0$, there is a GRS whose derivation height is $O(n^k)$.
\end{theorem}

\begin{proof}
Examples for the lower bound are proposed in~\cite{MSCS2013}.
For the upper bound, let
$$ K_\omega = \max \{ |\omega(n,e,m)| \mid (n,e,m) \in \nlabs \times \elabs \times \nlabs\} \mathrm{\ \ \ and\ \ \ } K_\pi = a \times |\elabs| \times K_\omega + b \times K_\eta$$

Then, adapting Lemma~\ref{le:weight}(\ref{le:weight:wG}) to the present context, we can state that $|\omega(G)| \leq K_\omega \times \size{\elabs} \times \size{G}^2$.
With Lemma~\ref{le:weight}(\ref{le:weight:eta}), we have $|\eta(G)| \leq K_\eta \times \size{G}$ and finally
$ |\pi(G)| \leq a \times  K_\omega \times \size{\elabs} \times \size{G}^2 + b \times K_\eta \times \size{G} \leq K_\pi \times \size{G}^2$.

Let $K_0 = \max_{i\in[1,k]} (K_{\pi_i})$.
Finally, let $K_E$ be the constant as given by Lemma~\ref{le:weight} for $w_0$, we define $K = \max(K_0,K_E)$.
Then, for all $i \leq k$, we have: $|\pi_i(G)| \leq K \times |G|^2$ and $|w_0(G)| \leq K \times |G|^2$.

Let $\kappa(G) = (|G|, w_0(G), \pi_1(G), \ldots, \pi_k(G))$. If $G \to G'$, then $\kappa(G) > \kappa(G')$.
Consider a sequence $G_1 \to G_2 \to \cdots$. For all graph $G_i$ of the sequence, $|G_i| \leq |G_1|$.
Due to previous equations, $\kappa(G_i)$ is ranging in $L = [0,|G_1|] \times [-K \times |G_1|^2,K\times |G_1|^2]^{k+1}$. Thus the result.
\end{proof}

\section{Conclusion}

The polynomial derivation height that we have proved in the last section can be reconsidered in the following way.
The example of a GRS working in $O(n^k)$ can be used as a clock.
Then, since each transition of a (non-size increasing) Turing-Machine can be easily simulated by graph rewriting, we can state that any {\sc ptime}-predicate can be simulated by a lexicographically weighted GRS (up to a polynomial reduction).
Since lexicographically weighted confluent GRS can be computed in polynomial time (each rewriting step can be simulated in linear time), it becomes clear that lexicographically weighted GRSs actually characterize {\sc ptime}.
This provides a precise description of the computational content of the method.

We have implemented a software |called {\sc grew} (\url{grew.loria.fr})| based on the Graph Rewriting definition presented in this article.
In~\cite{treebanking12}, the software was used to produce a semantically annotated version of the French Treebank; in this experiment, the system contains 34 modules and 571 rules and the corpus is constituted of 12\,000 sentences of length up to 100 words.
This experiment is a large scale application which shows that the proposed approach can be used in real-size applications.

As said earlier, despite the global non-confluence of the system, we can isolate subsets of rules that are confluent and use our system of modules to benefit from this confluence in implementation.
In our last experiment, 26 of our 34 modules are confluent, but confluence proofs are tedious.
We leave for further work the study of the local confluence of terminating GRS and the general study of confluence of Graph Rewriting Systems.

\bibliographystyle{eptcs}
\bibliography{grewbib}

\begin{thebibliography}{10}
\providecommand{\bibitemdeclare}[2]{}
\providecommand{\surnamestart}{}
\providecommand{\surnameend}{}
\providecommand{\urlprefix}{Available at }
\providecommand{\url}[1]{\texttt{#1}}
\providecommand{\href}[2]{\texttt{#2}}
\providecommand{\urlalt}[2]{\href{#1}{#2}}
\providecommand{\doi}[1]{doi:\urlalt{http://dx.doi.org/#1}{#1}}
\providecommand{\bibinfo}[2]{#2}

\bibitemdeclare{inproceedings}{Bohnet01}
\bibitem{Bohnet01}
\bibinfo{author}{B.~\surnamestart Bohnet\surnameend} \&
  \bibinfo{author}{L.~\surnamestart Wanner\surnameend} (\bibinfo{year}{2001}):
  \emph{\bibinfo{title}{On using a parallel graph rewriting formalism in
  generation}}.
\newblock In: {\sl \bibinfo{booktitle}{EWNLG '01: Proceedings of the 8th
  European workshop on Natural Language Generation}},
  \bibinfo{publisher}{Association for Computational Linguistics}, pp.
  \bibinfo{pages}{1--11}, \doi{10.3115/1117840.1117847}.

\bibitemdeclare{article}{MSCS2013}
\bibitem{MSCS2013}
\bibinfo{author}{G.~\surnamestart Bonfante\surnameend} \&
  \bibinfo{author}{B.~\surnamestart Guillaume\surnameend}
  (\bibinfo{year}{2013}): \emph{\bibinfo{title}{Non-size increasing Graph
  Rewriting for Natural Language Processing}}.
\newblock {\sl \bibinfo{journal}{{\em to appear in} Mathematical Structures for
  Computer Science}}.

\bibitemdeclare{inproceedings}{iwcs11}
\bibitem{iwcs11}
\bibinfo{author}{{G}. \surnamestart {B}onfante\surnameend},
  \bibinfo{author}{{B}. \surnamestart {G}uillaume\surnameend},
  \bibinfo{author}{{M}. \surnamestart {M}orey\surnameend} \&
  \bibinfo{author}{{G}. \surnamestart {P}errier\surnameend}
  (\bibinfo{year}{2011}): \emph{\bibinfo{title}{Modular Graph Rewriting to
  Compute Semantics}}.
\newblock In: {\sl \bibinfo{booktitle}{{IWCS} 2011}}, \bibinfo{address}{Oxford,
  UK}, pp. \bibinfo{pages}{65--74}.

\bibitemdeclare{inproceedings}{copestake}
\bibitem{copestake}
\bibinfo{author}{A.~\surnamestart Copestake\surnameend} (\bibinfo{year}{2009}):
  \emph{\bibinfo{title}{{\em Invited Talk:} {S}lacker {S}emantics: {W}hy
  {S}uperficiality, {D}ependency and {A}voidance of {C}ommitment can be the
  {R}ight {W}ay to {G}o}}.
\newblock In: {\sl \bibinfo{booktitle}{Proceedings of the 12th Conference of
  the European Chapter of the ACL (EACL 2009)}},
  \bibinfo{publisher}{Association for Computational Linguistics},
  \bibinfo{address}{Athens, Greece}, pp. \bibinfo{pages}{1--9}.

\bibitemdeclare{conference}{Crouch05}
\bibitem{Crouch05}
\bibinfo{author}{D.~\surnamestart Crouch\surnameend} (\bibinfo{year}{2005}):
  \emph{\bibinfo{title}{{Packed {R}ewriting for {M}apping {S}emantics to KR}}}.
\newblock In: {\sl \bibinfo{booktitle}{Proceedings of IWCS}}.

\bibitemdeclare{inproceedings}{echahed}
\bibitem{echahed}
\bibinfo{author}{R.~\surnamestart Echahed\surnameend} (\bibinfo{year}{2008}):
  \emph{\bibinfo{title}{Inductively Sequential Term-Graph Rewrite Systems}}.
\newblock In: {\sl \bibinfo{booktitle}{Proceedings of the 4th international
  conference on Graph Transformations}}, \bibinfo{series}{ICGT '08},
  \bibinfo{publisher}{Springer-Verlag}, \bibinfo{address}{Berlin, Heidelberg},
  pp. \bibinfo{pages}{84--98}, \doi{10.1007/978-3-540-87405-8\_7}.

\bibitemdeclare{inproceedings}{DBLP:conf/gg/GodardMMS02}
\bibitem{DBLP:conf/gg/GodardMMS02}
\bibinfo{author}{E.~\surnamestart Godard\surnameend},
  \bibinfo{author}{Y.~\surnamestart M{\'e}tivier\surnameend},
  \bibinfo{author}{M.~\surnamestart Mosbah\surnameend} \&
  \bibinfo{author}{A.~\surnamestart Sellami\surnameend} (\bibinfo{year}{2002}):
  \emph{\bibinfo{title}{Termination Detection of Distributed Algorithms by
  Graph Relabelling Systems}}.
\newblock In \bibinfo{editor}{A.~\surnamestart Corradini\surnameend},
  \bibinfo{editor}{H.~\surnamestart Ehrig\surnameend}, \bibinfo{editor}{H.-J.
  \surnamestart Kreowski\surnameend} \& \bibinfo{editor}{G.~\surnamestart
  Rozenberg\surnameend}, editors: {\sl \bibinfo{booktitle}{ICGT}}, {\sl
  \bibinfo{series}{Lecture Notes in Computer Science}} \bibinfo{volume}{2505},
  \bibinfo{publisher}{Springer}, pp. \bibinfo{pages}{106--119},
  \doi{10.1007/3-540-45832-8\_10}.

\bibitemdeclare{inproceedings}{Hyvonen84}
\bibitem{Hyvonen84}
\bibinfo{author}{E.~\surnamestart Hyv{\"o}nen\surnameend}
  (\bibinfo{year}{1984}): \emph{\bibinfo{title}{{S}emantic {P}arsing as {G}raph
  {L}anguage {T}ransformation - a {M}ultidimensional {A}pproach to {P}arsing
  {H}ighly {I}nflectional {L}anguages}}.
\newblock In: {\sl \bibinfo{booktitle}{COLING}}, pp. \bibinfo{pages}{517--520},
  \doi{10.3115/980491.980601}.

\bibitemdeclare{inproceedings}{Jijkoun07}
\bibitem{Jijkoun07}
\bibinfo{author}{V.~\surnamestart Jijkoun\surnameend} \&
  \bibinfo{author}{M.~\surnamestart de~Rijke\surnameend}
  (\bibinfo{year}{2007}): \emph{\bibinfo{title}{Learning to Transform
  Linguistic Graphs}}.
\newblock In: {\sl \bibinfo{booktitle}{Second {W}orkshop on {T}ext{G}raphs:
  {G}raph-{B}ased {A}lgorithms for {N}atural {L}anguage {P}rocessing,
  Rochester, {NY}, {USA}}}.

\bibitemdeclare{inproceedings}{KB70}
\bibitem{KB70}
\bibinfo{author}{D.E. \surnamestart Knuth\surnameend} \& \bibinfo{author}{P.B.
  \surnamestart Bendix\surnameend} (\bibinfo{year}{1970}):
  \emph{\bibinfo{title}{{S}imple word problems in universal algebras}}.
\newblock In \bibinfo{editor}{J.~\surnamestart Leech\surnameend}, editor: {\sl
  \bibinfo{booktitle}{{C}omputational problems in abstract algebra}},
  \bibinfo{address}{{P}ergamon}, pp. \bibinfo{pages}{263--277}.

\bibitemdeclare{article}{N42}
\bibitem{N42}
\bibinfo{author}{M.~\surnamestart Newman\surnameend} (\bibinfo{year}{1942}):
  \emph{\bibinfo{title}{On Theories With a Combinatorial Definition of
  "Equivalence"}}.
\newblock {\sl \bibinfo{journal}{Annals of Math.}}
  \bibinfo{volume}{43}(\bibinfo{number}{2}), pp. \bibinfo{pages}{223--243},
  \doi{10.2307/1968867}.

\bibitemdeclare{inproceedings}{treebanking12}
\bibitem{treebanking12}
\bibinfo{author}{G.~\surnamestart Perrier\surnameend} \&
  \bibinfo{author}{B.~\surnamestart Guillaume\surnameend}
  (\bibinfo{year}{2012}): \emph{\bibinfo{title}{{Semantic Annotation of the
  French Treebank with Modular Graph Rewriting}}}.
\newblock In \bibinfo{editor}{Jan \surnamestart Hajic\surnameend}, editor: {\sl
  \bibinfo{booktitle}{{META-RESEARCH Workshop on Advanced Treebanking, LREC
  2012 Workshop}}}, \bibinfo{organization}{META-NET},
  \bibinfo{address}{Istanbul, Turquie}.
\newblock \urlprefix\url{http://hal.inria.fr/hal-00760577}.

\bibitemdeclare{inproceedings}{Plump:1995:TGR:647676.731819}
\bibitem{Plump:1995:TGR:647676.731819}
\bibinfo{author}{D.~\surnamestart Plump\surnameend} (\bibinfo{year}{1995}):
  \emph{\bibinfo{title}{On Termination of Graph Rewriting}}.
\newblock In: {\sl \bibinfo{booktitle}{Proceedings of the 21st International
  Workshop on Graph-Theoretic Concepts in Computer Science}},
  \bibinfo{series}{WG '95}, \bibinfo{publisher}{Springer-Verlag},
  \bibinfo{address}{London, UK}, pp. \bibinfo{pages}{88--100},
  \doi{10.1007/3-540-60618-1\_68}.

\bibitemdeclare{article}{Plump98a}
\bibitem{Plump98a}
\bibinfo{author}{D.~\surnamestart Plump\surnameend} (\bibinfo{year}{1998}):
  \emph{\bibinfo{title}{Termination of Graph Rewriting is Undecidable}}.
\newblock {\sl \bibinfo{journal}{Fundamenta Informaticae}}
  \bibinfo{volume}{33}(\bibinfo{number}{2}), pp. \bibinfo{pages}{201--209},
  \doi{10.3233/FI-1998-33204}.

\bibitemdeclare{proceedings}{dpo_spo}
\bibitem{dpo_spo}
\bibinfo{editor}{G.~\surnamestart Rozenberg\surnameend}, editor
  (\bibinfo{year}{1997}): \emph{\bibinfo{title}{Handbook of Graph Grammars and
  Computing by Graph Transformations, Volume 1: Foundations}}.
  \bibinfo{publisher}{World Scientific}.

\bibitemdeclare{book}{elements}
\bibitem{elements}
\bibinfo{author}{L.~\surnamestart Tesni\`ere\surnameend}
  (\bibinfo{year}{1959}): \emph{\bibinfo{title}{El\'ements de syntaxe
  structurale}}.
\newblock \bibinfo{publisher}{Librairie C. Klincksieck, Paris}.

\end{thebibliography}

\end{document}